\documentclass[journal, onecolumn]{IEEEtran}

\usepackage[ruled,vlined]{algorithm2e}

\usepackage{amsmath}
\usepackage{amsthm}
\usepackage{amsfonts}
\usepackage{amssymb}
\usepackage{bm}
\usepackage{boxedminipage}
\usepackage{cases}
\usepackage[usenames]{color}
\usepackage{array,color}
\usepackage{colortbl}
\usepackage{enumerate}
\usepackage{epsfig}
\usepackage{extarrows}
\usepackage{xcolor}
\usepackage[hidelinks]{hyperref}

\usepackage{booktabs}
\usepackage{multirow}
\usepackage{makecell}

\usepackage{graphicx}
\usepackage[labelformat=simple]{subcaption}

\newtheorem{remark}{\textbf{Remark}}

\newtheorem{proposition}{\textbf{Proposition}}

\graphicspath{{figures/}}


\begin{document}
\IEEEoverridecommandlockouts

\title{FedNC: A Secure and Efficient Federated Learning Method with Network Coding}

\author{\IEEEauthorblockN{
            Yuchen Shi\IEEEauthorrefmark{2}, Zheqi Zhu\IEEEauthorrefmark{2}, Pingyi Fan\IEEEauthorrefmark{1}\IEEEauthorrefmark{2},~\IEEEmembership{Senior Member,~IEEE,}\\ Khaled B. Letaief\IEEEauthorrefmark{3},~\IEEEmembership{Fellow,~IEEE,} and Chenghui Peng\IEEEauthorrefmark{4}}
        
	    \IEEEauthorblockA{
            \IEEEauthorrefmark{2}Dept. of Electronic Engineering, BNRist, Tsinghua University.\\
		    \IEEEauthorrefmark{3}Dept. of Electronic and Computer Engineering, HKUST.\\
            \IEEEauthorrefmark{4}Wireless Technology Laboratory, Huawei Technologies.\\
            \IEEEauthorrefmark{1}Email: fpy@tsinghua.edu.cn} 
        
        \vspace*{-7mm}

        \thanks{\IEEEauthorrefmark{1}Pingyi Fan is the corresponding author. This work was supported by the National Key Research and Development Program of China (Grant NO.2021YFA1000500(4)).}
        \thanks{\IEEEauthorrefmark{3}K. B. Letaief was supported in part by the Hong Kong Research Grant Council under Grant No. 16208921.}
        \thanks{Accepted by 2024 IEEE Wireless Communications and Networking Conference (WCNC) as a conference paper.}
        }

\maketitle
\IEEEpeerreviewmaketitle

\begin{abstract}
	
Federated Learning (FL) is a promising distributed learning mechanism which still faces two major challenges, namely privacy breaches and system efficiency. In this work, we reconceptualize the FL system from the perspective of network information theory, and formulate an original FL communication framework, FedNC, which is inspired by Network Coding (NC). The main idea of FedNC is mixing the information of the local models by making random linear combinations of the original parameters, before uploading for further aggregation. Due to the benefits of the coding scheme, both theoretical and experimental analysis indicate that FedNC improves the performance of traditional FL in several important ways, including security, efficiency, and robustness. To the best of our knowledge, this is the first framework where NC is introduced in FL. As FL continues to evolve within practical network frameworks, more variants can be further designed based on FedNC.

\end{abstract}


\section{Introduction} \label{sec:1_intro}
\IEEEPARstart{D}{ata} privacy and system efficiency are becoming more and more important in machine learning with the advent of the Big Data era. As a promising distributed learning framework, Federated Learning (FL) is a powerful learning mechanism where local clients collaboratively train a model with the help of a central server while keeping the data decentralized \cite{kairouz2021advances}. Unlike centralized learning that requires all data to be pooled for training, FL only needs the interaction of model parameters (or gradients) while the training data remains localized, which greatly improves the security and availability of the system. 

Since its introduction in 2016 \cite{mcmahan2017communication}, FL has taken on a crucial role in the 6th Generation or 6G system as the combination of distributed Artificial Intelligence (AI) and communication networks has been receiving increasing attention.
Specifically, FL demonstrates its great potential in deep learning, naturally fitting the decentralized structure of multi-user networks \cite{wan2023global} and multi-agent reinforcement learning systems \cite{zhu2021federated, wu2022mobility}. In addition, various frameworks have been derived from the classic FL architecture, such as decentralized FL \cite{lalitha2018fully}, hierarchical FL \cite{liu2022hierarchical}, and layer-wise pruning FL \cite{zhu2023fedlp}. In the future, FL can also be considered in physical layer cases, i.e., multiple-input-multiple-output (MIMO) relay networks \cite{xiong2014optimal} and energy harvesting technology \cite{lu2018global}.

However, despite the fact that FL was originally proposed to protect the data privacy, there is still the possibility of exposing important information during the communication between the local clients and the central server. For model aggregations, clients generally pass the local model parameters to the server over an open channel. If a malicious attack occurs during the transmission, the attacker may be able to steal the model parameters and obtain raw information from the clients, which may include hospital patient data, bank deposit information, company trade secrets, etc. 

The communication bottleneck is also another major challenge in FL systems. The wireless open channel and numerous participating clients, along with the heterogeneity of local data and training devices, amplify communication costs and instability compared to traditional distributed setups with stable wired connections. These factors result in reduced efficiency and robustness. Moreover, a failed model upload from a client forces the FL system to wait for re-upload, wasting time, or to discard parameters, risking errors due to data heterogeneity. 

Given the above, it is clear that privacy breaches and system efficiency are two major challenges that need to be addressed urgently in FL.


\subsection{Related Works and Motivations}
Numerous researches have been investigated to enhance the security of FL systems in the literature. \textit{Differential Privacy} (DP) \cite{abadi2016deep} introduces the controllable noise into raw data, aiming to protect the privacy of individual data. Geyer \textit{et al.} \cite{geyer2017differentially} proposed an algorithm for client sided DP preserving federated optimization via the Gaussian noise mechanism. Agarwal \textit{et al.} \cite{agarwal2018cpsgd} provided convergence guarantees for DP under both Binomial mechanism and Gaussian mechanism. 
However, all these approaches ensure a certain level of security at the expense of accuracy. Moreover, \textit{homomorphic encryption} \cite{acar2018survey} can also protect data privacy in limited scenarios by exchanging encrypted parameters during the learning process, although the extra encryption and approximation will result in trade-offs between accuracy and security. \textit{Secure multiparty computation} \cite{yao1986generate} enables multiple agents to establish a secure function that guarantees no information leakage except its input and output, but due to the excessive communication and computational costs, along with the necessity for carefully designed protocols for different tasks, it is not suitable for large-scale machine learning.

Meanwhile, \textit{coded computation} can also be used to improve the robustness and efficiency of FL. Prakash \textit{et al.} \cite{prakash2020coded} proposed a novel framework, CodedFedL, which adds coding schemes into federated system for mitigating stragglers and speeding up the training procedure. Anand \textit{et al.} \cite{anand2021differentially} improved the data privacy by adding Gaussian noise to the coded data, but they failed to exploit the local computational resources for fast training as the gradient computation on the server is restricted. Sun \textit{et al.} \cite{sun2022stochastic} proposed a stochastic coded FL framework which involved random linear combination of the local data, sharing the similar idea with CodedFedL. However, the problem of coded FL system lies in the necessity to share parity data of local clients to the central server, thereby introducing supplementary communication costs. It also imposes increased demands on the computational resources of the central server, not only requiring it to perform parameter aggregation, but also to compute partial gradient from the composite parity data. 

Inspired by the mechanism of the coded federated systems as well as the similar topology between FL system and network structure, we first combine the FL with \textit{Network Coding} (NC) and propose FedNC, a novel FL framework that enhances both the security and robustness of the federated system without losing accuracy and efficiency.
Ahlswede \textit{et al.} \cite{ahlswede2000network} showed that NC achieves maximal throughput of a multicast network, improving the network's efficiency and scalability, while Li \textit{et al.} \cite{li2003linear} proved that Linear Network Coding (LNC) is sufficient to achieve the max-flow from the source to destinations in multicast problems, so the receivers merely decode linearly independent combinations to obtain the original packets, not requiring persistently stable channels, thereby enhancing the robustness. Furthermore, LNC has the added benefit of security, since the attackers can only decode the packets if they have received a sufficient number of linearly independent encoded packets.
As a result, NC provides advantages aligning well with FL's challenges.

\subsection{Contributions and Paper Organization}

The contributions of this work are summarized as follows:

\begin{itemize}
    \item We propose an original FL system, FedNC, which is inspired by distributed network information theory. FedNC alleviates several difficulties in conventional FL systems, including information security and communication bottlenecks. To the best of our knowledge, this is the first application of NC in the FL scenario.
    \item We design a learning algorithm for FedNC, illustrate the effectiveness of FedNC in preventing privacy breaches, and put forward its potential advantages in network throughput, system robustness and efficiency. Meanwhile, the estimated upper bound for the error probability of the transmission in FedNC is also given.
    \item We conduct numerical experiments to evaluate the performance of FedNC. The outcomes suggest that FedNC significantly increases the robustness and efficiency compared to traditional FL systems, especially in non-iid (independent and identically distributed) cases and large-scale situations.
\end{itemize}

The rest of this paper is organized as follows. Section \ref{sec:2_pre} introduces the basic definitions and underlying concepts of FL and NC. In Section \ref{sec:3_fednc} we illustrate the framework and algorithm of FedNC, explain the benefits that NC brings to FL system in depth. Section \ref{sec:4_numerical} presents the numerical experiment results, as well as related performance evaluations and other discussions. Finally, we conclude this work in Section \ref{sec:5_conclusion}.

\section{Preliminaries} \label{sec:2_pre}
In this section, we introduce key definitions and preliminaries including basic concepts of FL and NC.
\subsection{Federated Learning Problem}
The original FL problem involves learning a global model with numerous decentralized data from different local clients. In particular, suppose there is a typical FL system with $N$ clients, each client $k$ trains a local model $w_k$ with its local dataset $\mathcal{D}_k$, where $k=1,2,\dots,N$, the goal is to minimize the following distributed optimization problem,
\begin{equation}
    \min_w f(w) \quad \text{where} \quad f(w)\triangleq \sum\limits^N_{k=1} p_kf_k(w)
\end{equation}
where $p_k$ is the weight of the $k$-th client such that $p_k \geq 0$ and $\sum_k p_k =1$, $f(w)$ is the global objective function and $f_k(w)$ is the local objective function. In order to aggregate the data, classic FL procedure $\mathtt{FedAvg}$ \cite{mcmahan2017communication} randomly select $K$ clients to form a participator set $\mathcal{P}_t$ on each communication round, and train the local models with their own data. These $K$ clients upload their model parameters to the central server for model aggregation after the local training is finished. Once the global model is updated, the new global model will be downloaded by selected clients in next round for further training. 


\subsection{Random Linear Network Coding}
In LNC, we assume that there are $K$ packets $\{P_1$, $P_2$, $\dots$, $P_K\}$ that are coded together in a single transmission with each packet having the length of $L$ bits. Every consecutive $s$ bits from the start position form a \textit{symbol} (i.e., a symbol is a byte when $s=8$). Thus, each packet consists of $L/s$ symbols \cite{fragouli2006network}. Symbols are the basic unit of linear combinations which are taken over a finite field (Galois field) $\mathbb{F}_{2^s}$ in order to guarantee the length invariance when performing the linear transformation in binary. Each packet $P_k$ over the network is associated with \textit{coding coefficients} $\alpha_{ik}\in\mathbb{F}_{2^s}$, where $i=1,2,\dots, K$, and each \textit{encoded packet} $C_i$ is given by
\begin{equation}\label{eq_LNC}
    C_i=\sum\limits_{k=1}^K\alpha_{ik}P_k
\end{equation}
Since packets are encoded symbol by symbol, the summation occurs for every symbol position. That is,
\begin{equation}
    C_i^m=\sum\limits_{k=1}^K\alpha_{ik}P_k^m
\end{equation}
where $m=1,2,\dots,L/s$ is the symbol index. Without loss of generality, we use Eqn. (\ref{eq_LNC}) in the subsequent analysis for simplicity. 
It is worth mentioning that to address the difficulty of representing any real numbers in finite field, a common approach is to adapt the training process to make the (over)underflows controllable \cite{kairouz2021advances}, such as by operating on quantization \cite{zhu2023towards}. However, this is beyond the scope of this study.

During the transmission, $C_i$ is sent together with the corresponding \textit{encoded vector} $\boldsymbol{a_i}=(\alpha_{i1},\alpha_{i2},\dots,\alpha_{iK})\in\mathbb{F}_{2^s}^K$.
Consider a node that has received $K$ tuples $(\boldsymbol{a_i}, C_i)$ where $i=1,2,\dots,K$, then these encoded vectors will form an \textit{encoded matrix} $A=(\boldsymbol{a_1}, \boldsymbol{a_2}, \dots, \boldsymbol{a_K})^T\in\mathbb{F}_{2^s}^{K\times K}$.
Retrieving original packets is translated into solving the linear system $C=AP$, where $C=(C_1, C_2, \dots, C_K)^T$ and $P=(P_1, P_2, \dots, P_K)^T$. One can only get a well-posed problem if $A$ is an invertible matrix, i.e., $\{\boldsymbol{a_1}, \boldsymbol{a_2}, \dots, \boldsymbol{a_K}\}$ are linearly independent.

The major problem in LNC design is how to select the linear combinations of each node in the network.
An easy approach is to randomly select the coding coefficients over the finite field $\mathbb{F}_{2^s}$. Random Linear Network Coding (RLNC) \cite{ho2003randomized} provides a powerful and useful application for LNC, eventhough there is a certain probability of selecting linearly dependent combinations. It is proved that the probability that all receivers can obtain the original packets increases with the increase of the finite field size $s$. In fact, numerical experiments indicate that the error probability becomes negligible even for a small field size (i.e., $s=8$).

\begin{figure}[htbp]
    \centering
    \includegraphics[width = 0.8\columnwidth]{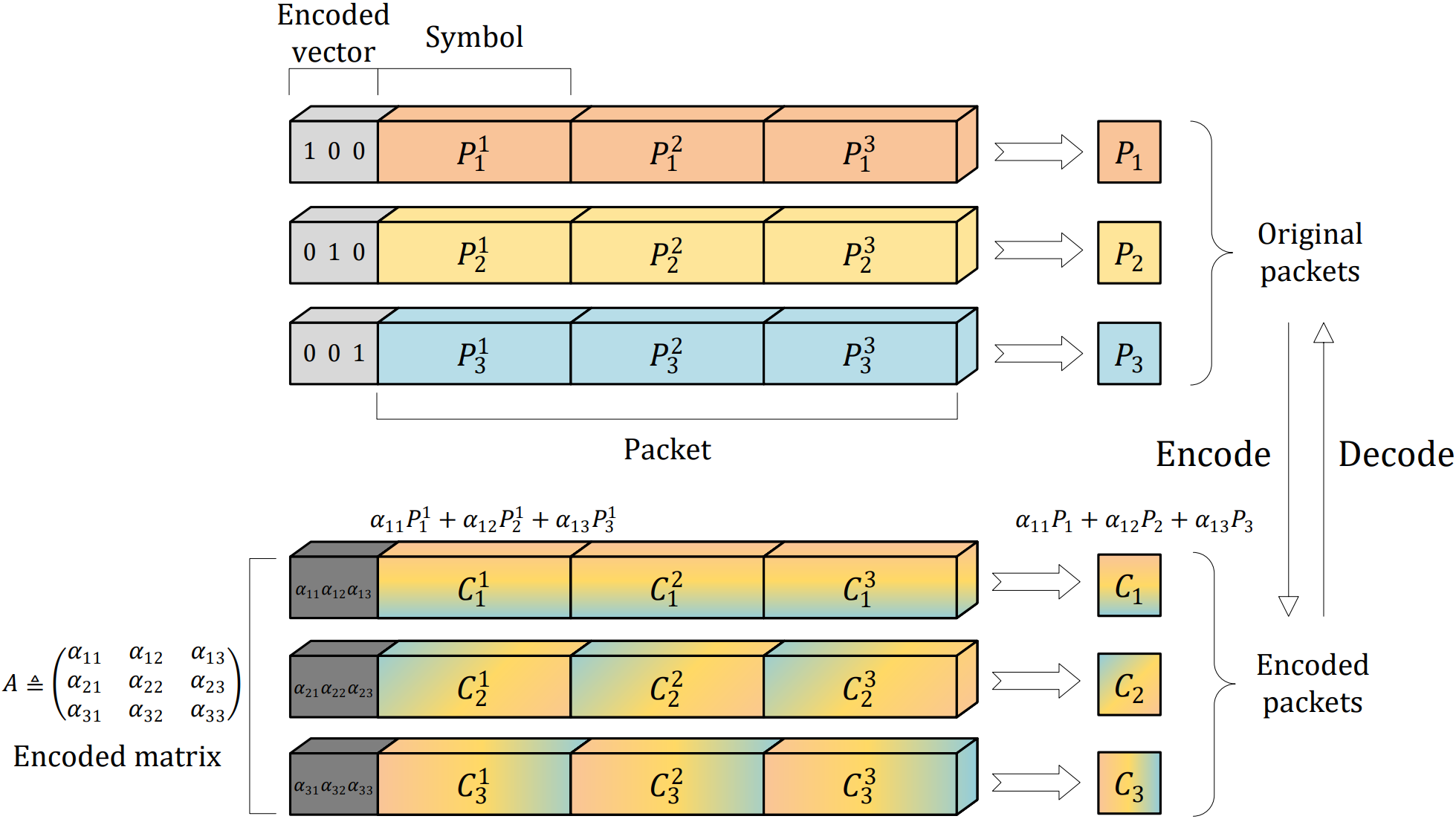}
    \caption{The schematic of the coding process in LNC.}
    \label{fig:nc}
\end{figure}


\section{FedNC: Framework and Algorithm} \label{sec:3_fednc}

In this section, we formally present the framework of FedNC, design the corresponding algorithm implementation, and demonstrate the advantages of FedNC over conventional FL methods.

As mentioned earlier, the main idea of FedNC is mixing the information of the local model parameters before aggregation. There are multiple ways to achieve the mixing. For instance, one can utilize the structure of hierarchical FL where local clients encode their parameters at trusted edge servers before uploading them to the central server. Alternatively, a fully decentralized FL machanism can also be employed, naturally incorporating encoding into the process as parameters are transmitted between users. It's assumed that the mixing procedure typically occurs in nearby cells or closed channels, ensuring that it remains immune to attacks or eavesdropping.


For simplicity, we consider the local parameters uploaded by each client as a packet. When the local training is finished, the participating clients will send their original packets for random linear combinations to get the encoded packets. Afterwards, the encoded packets will be uploaded to the central server, where they are decoded to get the original packets again to perform the subsequent model aggregation. Finally, the global model is updated at the central server and distributed to all local clients.


\begin{figure}[htbp]
    \centering
    \subcaptionbox{Traditional FL system\label{fig:flsystem}}[.48\columnwidth]
        {\includegraphics[width = .45\columnwidth]{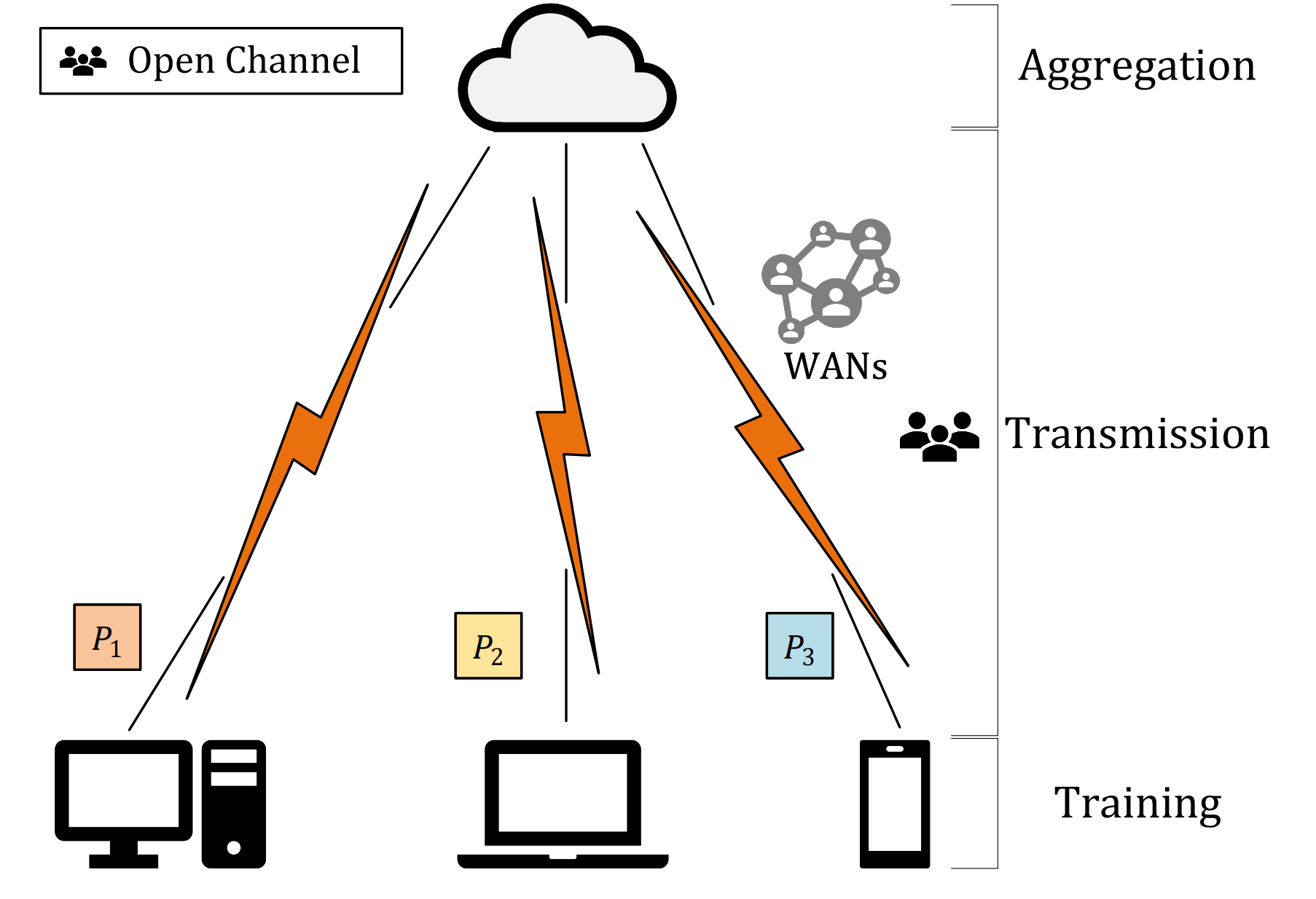}}
    \subcaptionbox{FedNC system\label{fig:fedncsystem}}[.48\columnwidth]
        {\includegraphics[width = .45\columnwidth]{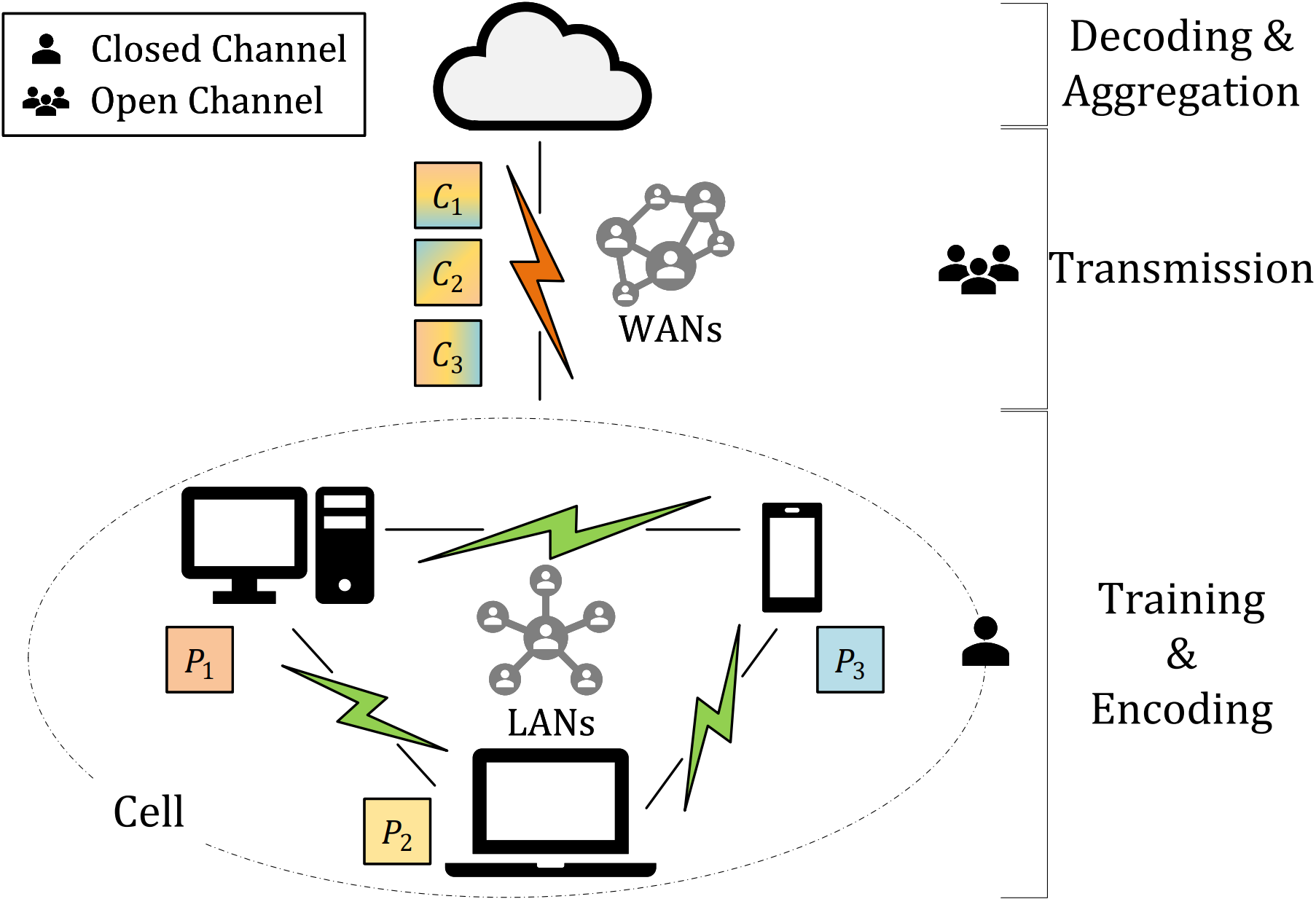}}
    \caption{Traditional FL and FedNC structures. Although both are vulnerable to attack threats during transmission over open channels, FedNC transmits encoded packets while traditional FL transmits original packets.}
\end{figure}

Assume there are $K$ clients involved in the FL communication round $t$, and let $w_k^{(t)}$ denote the local parameters (i.e., original packet) for client $k\in\mathcal{P}_t$, then we have
$$P=(w_1^{(t)}, w_2^{(t)}, \dots, w_K^{(t)})^T$$
According to RLNC, we randomly generate an encoded vector $\boldsymbol{a_i}$, multiplied by $P$ to get the encoded packets, 
\begin{equation}
    C_i=\boldsymbol{a_i}P=\sum\limits_{k=1}^K\alpha_{ik}w_k^{(t)}
\end{equation}
and then transmit them along with the encoded vectors to the central server in the form of encoded tuples $(\boldsymbol{a_i}, C_i)$, where $i=1,2,\dots,K$. Once the server obtains $K$ tuples, an encoded matrix $A$ is formed. If $A$ is invertible, the server can immediately decode them by \textit{Gaussian Elimination} ($\mathtt{GE}$) to get the original packets. However, if $A$ is a singular matrix, it cannot obtain $P$ with $K$ linear dependent equations. In this case, we proceed directly to the next round. All encoding and decoding processes are linearly transformed, so it does not significantly increase the consumption of the computing resources. The pseudo-code of $\mathtt{FedNC}$ is presented in Algorithm \ref{algo_fednc}.
\begin{algorithm}[htbp] 
    \SetAlgoLined
    Initialize the global model $w^{(0)}$\;
    \For{each round $t=1, \dots, T$}{
        $\mathcal{P}_t \gets$ (randomly select $K$ clients)\;
        \For{each client $k\in \mathcal{P}_t$ \textbf{in parallel}}{
        $w_k^{(t)}\gets \mathtt{local\_train}(w^{(t-1)}, \mathcal{D}_k)$\;
        }
        $P\gets [w_1^{(t)}, w_2^{(t)}, \dots, w_K^{(t)}]$\;
        \For{each $i=1, \dots, K$}{
        randomly generated $\boldsymbol{a_i}$\;
        $C_i\gets \boldsymbol{a_i}P$\;
        }
        $C\gets [C_1, C_2, \dots, C_K]$\;
        $A\gets [\boldsymbol{a_1}, \boldsymbol{a_2}, \dots, \boldsymbol{a_K}]$\;
        \eIf{matrix $A$ is invertible}{
        $\hat{P}=[\hat{w}_1^{(t)}, \hat{w}_2^{(t)}, \dots, \hat{w}_K^{(t)}]\gets \mathtt{GE}(A,C)$\;
        $w^{(t)}\gets \sum^K_{k=1} p_k\hat{w}_k^{(t)}$\;
        }{
        $w^{(t)}\gets w^{(t-1)}$\;
        }  
    }
    \caption{$\mathtt{FedNC}$}\label{algo_fednc}
   
\end{algorithm}


\subsection{Benefits of FedNC}
\subsubsection{Security}
FedNC mitigates privacy breaches by employing random linear network coding, which differs from traditional FL methods like $\mathtt{FedAvg}$, where original packets are vulnerable to eavesdropping during transmission. FedNC uses encoded packets for transmission, so the attacker must acquire enough linearly independent encoded packets to access the original data. Additionally, FedNC can be combined with other security methods to further enhance the security.



\subsubsection{Troughput Gain}
Coding technologies are not only used in one-to-one scenes, but also used in more complicated settings, namely multicast, multipath and even mesh systems. Consider a multicast acyclic graph, it has been proven that the maximum capacity can be achieved by performing RLNC over a sufficiently large field \cite{ho2003randomized}. Therefore, FedNC can help better share the available system resources.


\subsubsection{Robustness}
Intuitively, each packet without coding in traditional FL system is equally important. To obtain the parameters of all participators, we have to ensure that all channels are stable at all times. But with FedNC, no packet is irreplaceable. If any $K$ linearly independent combinations are provided, we will be able to recover the original packets. 


\subsubsection{Efficiency}
Unlike homomorphic encryption or multiparty computation that significantly increase communication costs, FedNC only requires the additional transmission of encoded tuples, which is negligible compared to the original parameter size. Additionally, FedNC doesn't share a portion of local data with the central server like the coded FL systems requiring, thus reducing its computational load. Furthermore, if the server cannot guarantee that each received packet comes from a different client, the efficiency of FedNC might even surpass that of $\mathtt{FedAvg}$, as described in Proposition \ref{pro:ccp}.

\begin{proposition} \label{pro:ccp}
Consider an FL system and assume that the central server would like to collect the packets from all clients in the set $\mathcal{P}_t$, where $|\mathcal{P}_t|=K$, and the number of packets for each client is the same. Assume that the server selects packets according to random sampling. Let $G$ be the total number of packets the server needs to obtain so as to collect all $K$ types of packets, then the expectation of $G$ satisfies:
\begin{equation}
    \mathbb{E}(G)=K\ln K+\gamma K+\frac{1}{2}+\mathcal{O}\left(\frac{1}{K}\right)
\end{equation}
where $\gamma\approx 0.577$ is the Euler–Mascheroni constant.
\end{proposition}
\begin{proof}
    Let $G_i$ measures the number of packets that need to be obtained to collect the $i$-th client's packet, we can express $G=\sum_{i=1}^{K}G_i$, $G_i \sim Ge(\frac{K-i+1}{K})$ where $Ge(\cdot)$ stands for geometric distribution. Since the expectation of a geometric random variable follows $\mathbb{E}(G_i)=\frac{K}{K-i+1}$, then we have:
    \begin{align}
        \mathbb{E}(G)&=\mathbb{E}(G_1)+\mathbb{E}(G_2)+\cdots+\mathbb{E}(G_K) \\
        &=\frac{K}{K}+\frac{K}{K-1}+\cdots+\frac{K}{1}=KH(K)
    \end{align}
    $H(K)$ refers to the $K$-th harmonic number:
    \begin{align}
        H(K)&=1+\frac12+\frac13+\cdots+\frac1K \\
        &\approx \ln K+\gamma +\frac{1}{2K}+\mathcal{O}\left(\frac{1}{K^2}\right)
    \end{align}
    where $\gamma\approx 0.577$ is the Euler–Mascheroni constant.
\end{proof}


\vspace{-8pt}
\begin{remark}
    Proposition \ref{pro:ccp} presents an adaptation for the FL scenario of a well-known conclusion in probability theory, known as the 'Coupon Collector's Problem', which describes a '\textit{collecting all coupons and win}' situation. 
    The advantage of FedNC comes into play when we model the channel as real and complex networks instead of point-to-point links. When we obtain packets from real networks such as the Internet, we do not select packets from specific sources, but receive all that we can. If we model the process of acquiring packets as random sampling, a specific packet will become rare for the receiver. In fact, a centralized gossip-based protocol achieves the information sharing in $\mathcal{O}(K)$ times, whereas a traditional FL system needs $\mathcal{O}(K\ln K)$ times, according to Proposition \ref{pro:ccp}. However, due to FedNC's random encoding mechanism, no more specific packets exist, so it is very likely that the original packets will be obtained as long as $K$ encoded packets are arbitrarily received, i.e., in $\mathcal{O}(K)$ times.
\end{remark}



\subsection{Error probability of transmission in FedNC}
Since the encoding coefficients are randomly generated in RLNC, there is a certain probability that the central server will not be able to decode the original packet successfully. For this problem, the authors of \cite{ho2003randomized} theoretically considered a feasible multicast system on a (possibly cyclic) network, and gave a lower bound on the probablity of successful recovery. We manage to generalize the conclusion to the FL scenario, and obtain a bound of the error probability during packets' decoding in FedFC as shown in the following proposition.

\begin{proposition} \label{pro:pe}
    For a FedNC system described in Algorithm \ref{algo_fednc}, assume that all the coding coefficients $\alpha_{ik}\in\mathbb{F}_{2^s}$. Suppose that the error probability for one communication round is denoted by $p_e$, then we have
    \begin{equation}
        p_e\leq 1-\left(1-\frac{1}{2^s}\right)^\eta
    \end{equation}
    where $\eta$ is the maximum number of links receiving signals with independent random coding coefficients in any set of links constituting a flow solution from the sources to the receiver.
\end{proposition}

\begin{proof}
    Consider a multicast connection system with unit delay links and independent or linearly correlated sources. Also suppose the probability that all the receivers can decode the source processes is $p$, then according to \cite{ho2003randomized}, we have 
    \begin{equation}
        p\geq\left(1-\frac{d}{2^s}\right)^\eta \quad \text{for} \quad s>\log_2d
    \end{equation}
    where $d$ is the number of receivers and $\eta$ is the number of links carrying random combinations of source processes and/or incoming signals. For a general FL system, we usually have one central server, i.e., $d=1$. Hence, 
    \begin{equation}
        p_e=1-\left.p\right|_{d=1}\leq 1-\left(1-\frac{1}{2^s}\right)^\eta
    \end{equation}

    \vspace{-2.5mm}
\end{proof}

\section{Numerical Results} \label{sec:4_numerical}
In this section, we present numerical experiments to evaluate the performance of FedNC in different circumstances.

\subsection{Experiment Setting}
We design the experiment as an image classification task under CIFAR-10 dataset. The total number of clients $N$ is set as \{100, 200\}, and the participation rate is set as \{0.1, 0.05\}, i.e., there are always 10 clients selected for model aggregation in each round. As for the process of obtaining packets by the server, we adopt the assumptions in Proposition \ref{pro:ccp}, i.e., the server does not know where the packet comes from, which is referred to here as '\textit{blind box effect}'. We consider that this is more consistent with the practical network setup, as well as the FedNC application scenario.

\subsubsection{Local Training}
We adopt the training model as a CNN based 6-Conv. layers neural network with batch normalization and max pooling operations. Adam algorithm is set as the stochastic gradient descent optimizer. For each participator, 5 epochs of local training are operated before sending the local parameters for subsequent encoding or aggregation.

\subsubsection{Data Splitting}
We execute the experiments under two data distribution settings, namely iid and mixed non-iid. For the iid setting, the complete training dataset is randomly assigned to all clients with each client holding data of uniform categories. For the mixed non-iid setting, the dataset is divided into shards with only one category, and each client possesses 2 shards, namely 2 categories, except for the 5\% iid parts.

\subsubsection{Encoding Setting}
Proposition \ref{pro:pe} illustrates that the finite field size $s$ and the number of links carrying random combinations $\eta$ will influence the probability of error, as well as the accuracy and convergence rate. Thus, we evaluate the performance of FedNC under different values of $s$ and $\eta$, and compare the gap between FedNC and classical FL ($\mathtt{FedAvg}$) under different scales of FL.

\subsection{Performance Evaluation}
Firstly, we focus on the test performance of FedNC with different values of $s$ and $\eta$. As presented in Fig. \ref{fig:K[100]} and Table \ref{tab:my_label}, the lower the error probability, the higher the test accuracy and also the faster the convergence rate, which is in line with our analysis. For the iid case, the blind box effect is not obvious due to the fact that each client possesses uniform data. Despite the fact that the convergence of FedNC is a bit slower due to decoding failures, one can observe that FedNC converges to the same optimum with $\mathtt{FedAvg}$ since the final test results are very close. Nevertheless, it is worth noting that for the mixed non-iid case, FedNC has very obvious advantages over classical FL, e.g., it converges faster and reaches a better accuracy. On the one hand, the non-iid property exacerbates the blind box effect of $\mathtt{FedAvg}$, making the outcome less accurate. On the other hand, thanks to the random coding characteristics of FedNC, the server can use packets from different clients for every successful aggregation. As a result, the coding mechanism significantly alleviates the drawbacks caused by the blind box effect while enhancing both the robustness and the efficiency.
\begin{figure}[htbp]
    \centering
    \subcaptionbox{iid splitting\label{fig:iid_K[100]}}[.49\columnwidth]
        {\includegraphics[width = .45\columnwidth]{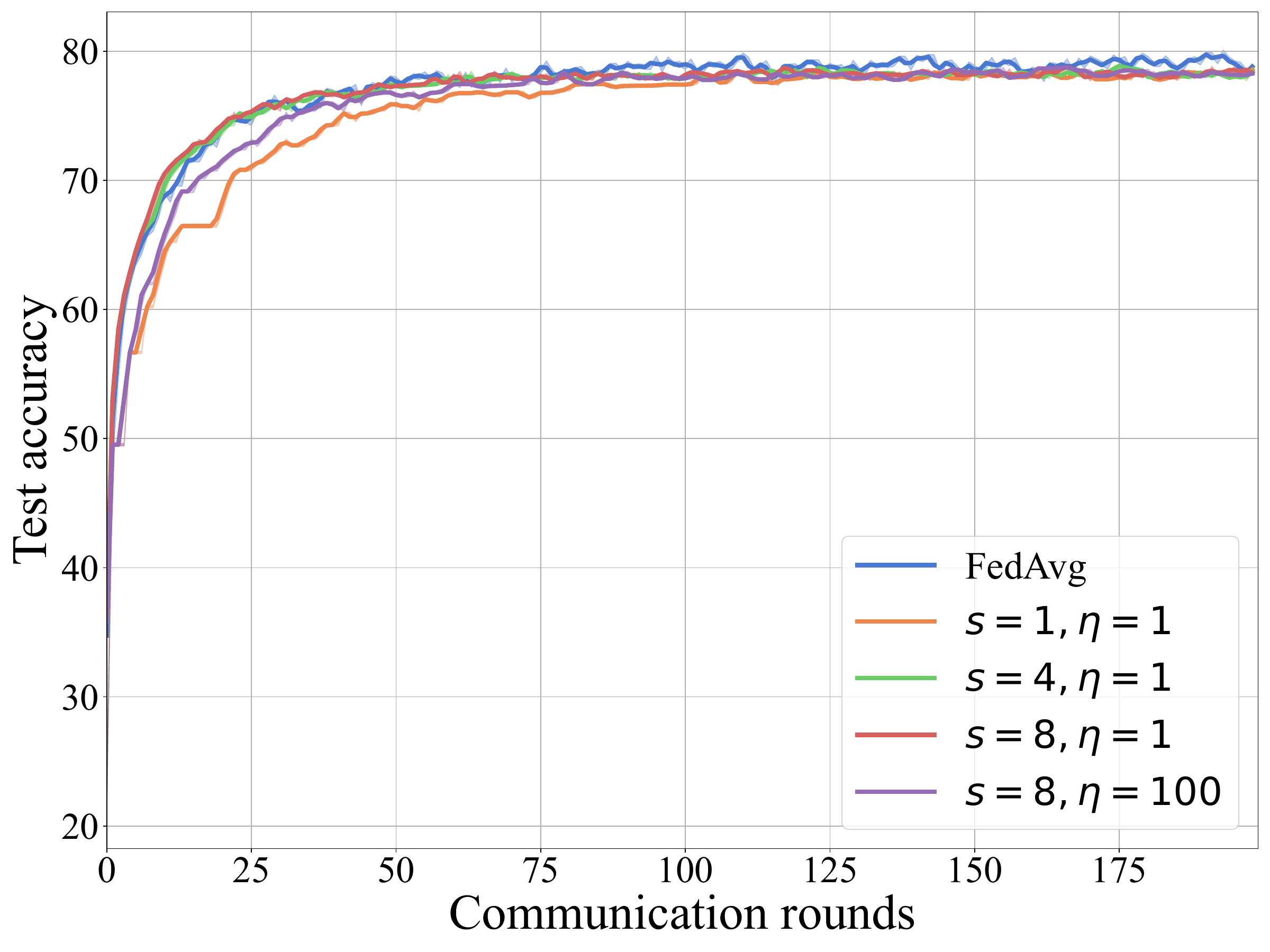}}
    \subcaptionbox{Mixed non-iid splitting\label{fig:miid_K[100]}}[.49\columnwidth]
        {\includegraphics[width = .45\columnwidth]{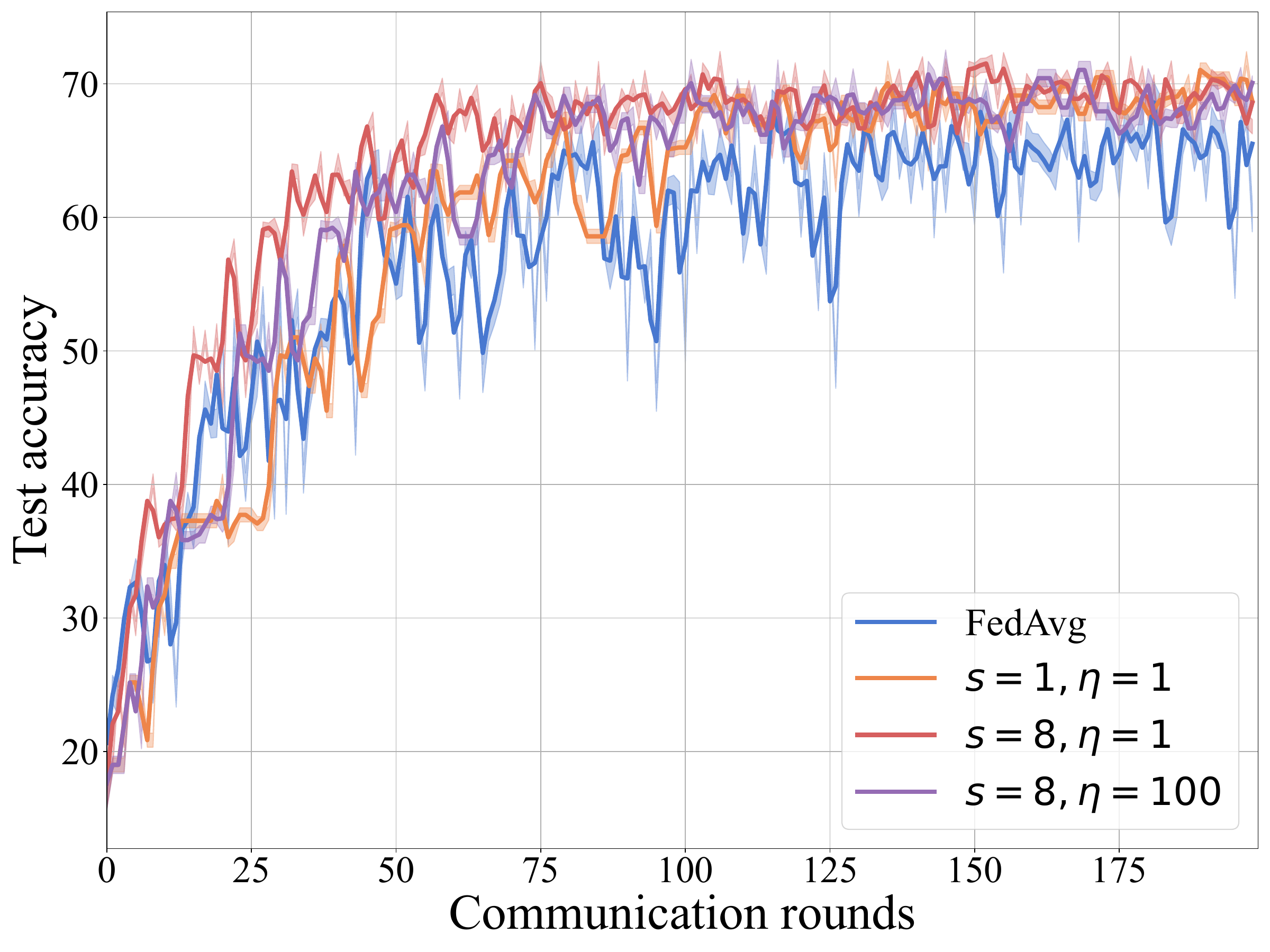}}
    \caption{Comparisons between classical FL ($\mathtt{FedAvg}$) and FedNC with numerous settings of field size $s$ and number of links $\eta$ under two different data splitting, where $N=100$.\label{fig:K[100]}}
\end{figure}


\begin{table}[htbp]
    \centering
    \begin{tabular}{cccc}
        \toprule
        \multirow{2}*{Schemes}  &  \multirow{2}*{Error Probability}  &  \multirowcell{2}{Accuracy (\%) \\ iid / mixed non-iid}  \\
        \\
        \midrule
        FedAvg                  &             -                     & \textbf{79.01} / 64.61  \\
        \midrule
        $s=1, \eta=1$           &            0.5                    & 78.12 / 68.93   \\
        $s=4, \eta=1$           &          0.0625                   & 78.25 / 69.29   \\
        $s=8, \eta=1$           &          0.0039                   & 78.26 / \textbf{69.42}   \\
        $s=8, \eta=100$         &          0.3239                   & 78.32 / 68.39   \\
        \bottomrule
    \end{tabular}
    \caption{Numerical results of the error probability and test accuracy under different settings.}
    \label{tab:my_label}
\end{table}

We also explore the results under different settings of FL system scale. It is necessary that we change the total number of clients and the participation rate together so that the size of the participator set $K$ remains the same. Fig. \ref{fig:K[100,200]} shows an obvious phenomenon that with the scale increasing (from $N=100$ to $200$), the blind box effect becomes more dominant for both data splitting, as the test accuracy declines for $\mathtt{FedAvg}$. However, it is notable that with lower participation rates, the advantage of FedNC becomes more evident, as the test accuracy of FedNC decreases less than classical FL under both data distributions, while converging even faster. This improvement results from the fact that the central server receives linear combinations in FedNC. Thus, on average there are more clients participating in the aggregation per round, and this superiority becomes more apparent when the participation rate is low. The experimental results are instructive because they highlight the advantages of FedNC in large-scale FL training, which is more indicative of practical systems.

\begin{figure}[htbp]
    \centering
    \subcaptionbox{iid splitting\label{fig:iid_K[100,200]}}[.49\columnwidth]
        {\includegraphics[width = .45\columnwidth]{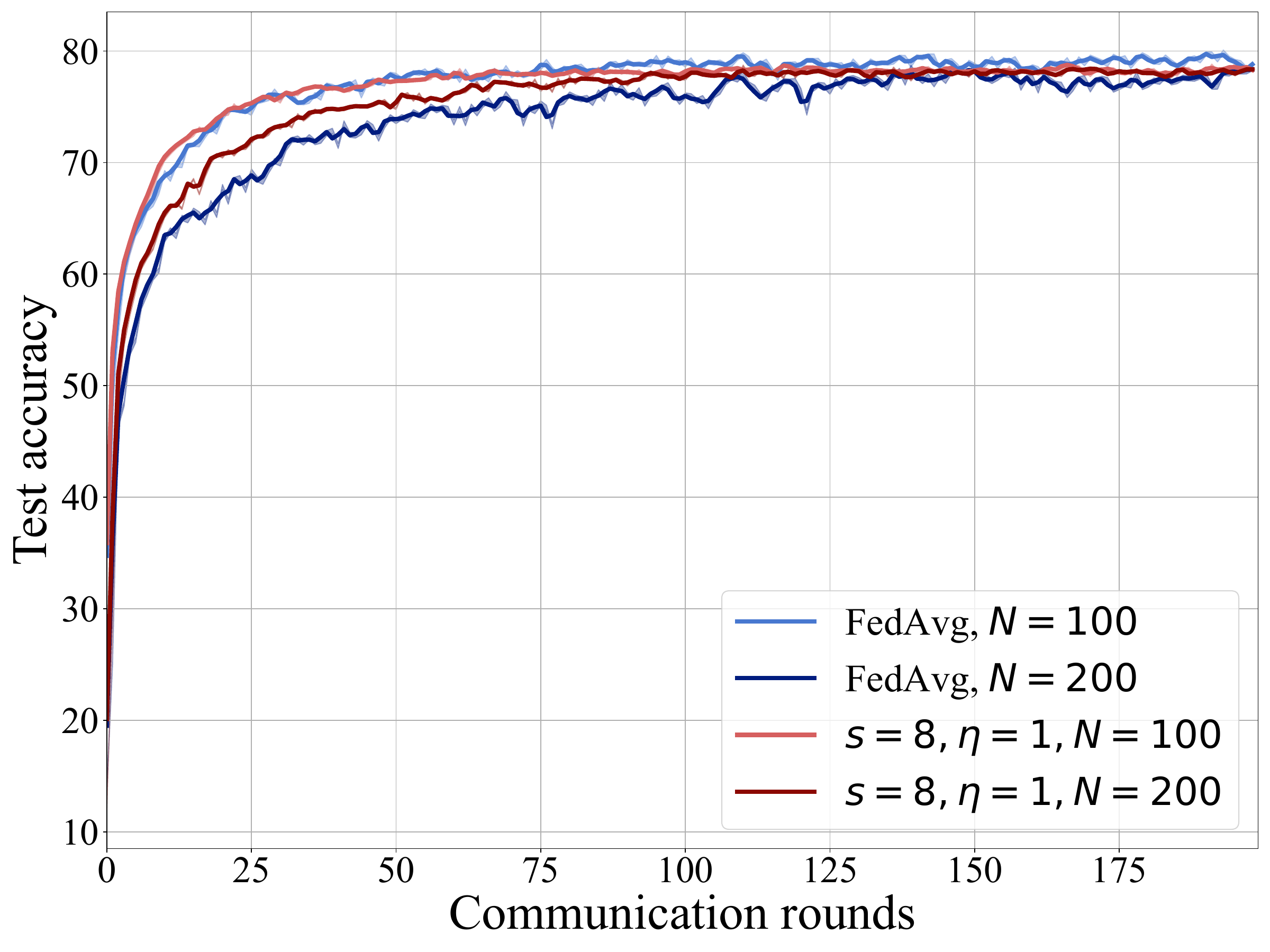}}
    \subcaptionbox{Mixed non-iid splitting\label{fig:miid_K[100,200]}}[.49\columnwidth]
        {\includegraphics[width = .45\columnwidth]{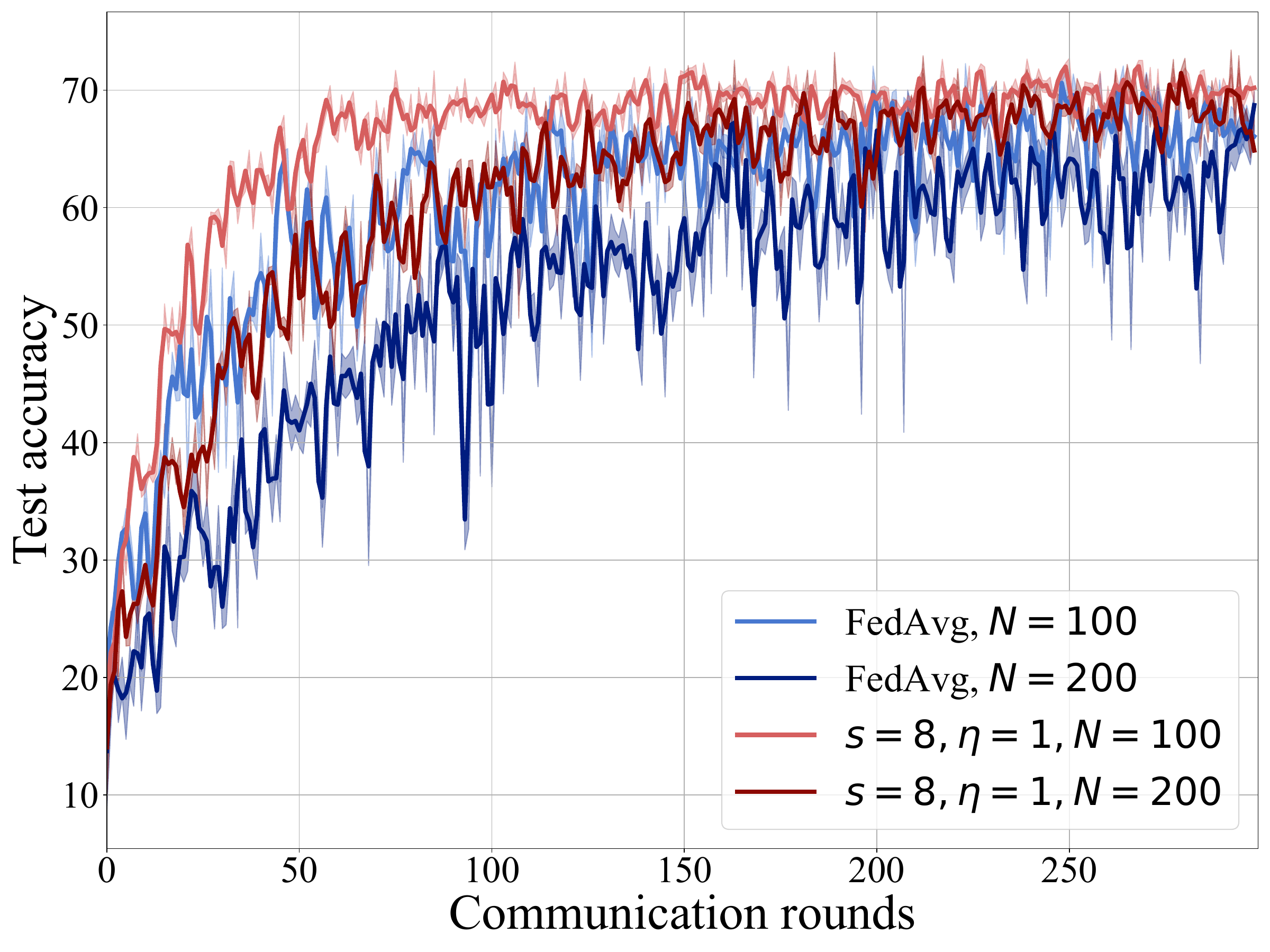}}
    \vspace{-1mm}
    \caption{Comparisons between classical FL ($\mathtt{FedAvg}$) and FedNC ($s=1, \eta=8$) under different system scales.\label{fig:K[100,200]}}
\end{figure}

\section{Conclusion} \label{sec:5_conclusion}
In this work, we considered NC in the context of FL, and proposed an original FL framework, FedNC, which encodes the packets before transmission to optimize security and efficiency. We designed the basic structure and algorithm for FedNC, illustrated several advantages of FedNC over the conventional FL, and theoretically portrayed the blind box effect that conventional FL suffers. Estimation for the error probability's upper bound of FedNC was also derived. The experimental results demonstrated that FedNC can significantly improve the robustness and efficiency compared to traditional FL systems, especially in non-iid cases and large-scale situations.

In the future, FedNC based works can be developed in several aspects. Firstly, the analysis of the convergence rate will be an important extension. Secondly, the theoretical description of FedNC's security guarantees is worth studying. Finally, richer experiments are needed to further explore FedNC's potentials for system efficiency and robustness.

\bibliographystyle{IEEEtran}
\bibliography{IEEEabrv, ref}

\end{document}